\documentclass[conference]{IEEEtran}
\IEEEoverridecommandlockouts

\usepackage{cite}
\usepackage{amsmath,amssymb,amsfonts}
\usepackage{amsthm}
\usepackage{graphicx}
\usepackage{textcomp}
\usepackage{xcolor}

\usepackage{float}
\usepackage[caption=false]{subfig} 
\usepackage[ruled,linesnumbered]{algorithm2e} 
\usepackage{multirow}
\usepackage{tabularx}
\usepackage{booktabs}
\usepackage{xcolor}
\allowdisplaybreaks[4]

\newtheorem{theorem}{Theorem}

\newtheorem{definition}{Definition}

\def\BibTeX{{\rm B\kern-.05em{\sc i\kern-.025em b}\kern-.08em
 T\kern-.1667em\lower.7ex\hbox{E}\kern-.125emX}}

\begin{document}

\title{Rethinking the Transferable Adversarial Attacks and Robust Defense in Federated Learning}

\author{\IEEEauthorblockN{Zuobin Xiong}
\IEEEauthorblockA{\textit{University of Nevada Las Vegas} \\
Las Vegas, Nevada \\
zuobin.xiong@unlv.edu}
\and
\IEEEauthorblockN{Deval Mukherjee}
\IEEEauthorblockA{\textit{University of Nevada Las Vegas} \\
Las Vegas, Nevada\\
mukhed2@unlv.nevada.edu}
\and
\IEEEauthorblockN{Homook Cho}
\IEEEauthorblockA{\textit{Cyber Security Research Center at KAIST} \\
Daejeon, South Korea\\
chmook79@kaist.ac.kr}
\and
\IEEEauthorblockN{Wei Li}
\IEEEauthorblockA{\textit{Georgia State University} \\
Atlanta, Georgia\\
wli28@gsu.edu}
}

\author{\IEEEauthorblockN{
		Zuobin Xiong\IEEEauthorrefmark{1},
		Deval Mukherjee\IEEEauthorrefmark{1},
		Homook Cho\IEEEauthorrefmark{2} and
		Wei Li\IEEEauthorrefmark{3}}
		\IEEEauthorblockA{
		\IEEEauthorrefmark{1}Department of Computer Science, University of Nevada Las Vegas, Las Vegas, USA\\
		\IEEEauthorrefmark{2}Cyber Security Research Center at KAIST,
Daejeon, South Korea\\
\IEEEauthorrefmark{3}Department of Computer Science, Georgia State University,
Atlanta, Georgia, USA\\
		\IEEEauthorrefmark{1} {\it zuobin.xiong@unlv.edu, mukhed2@unlv.nevada.edu;
		\IEEEauthorrefmark{2} {\it chmook79@kaist.ac.kr}}
		}
        \IEEEauthorrefmark{3} {\it wli28@gsu.edu}
        }

\maketitle

\begin{abstract}
The development of federated learning (FL) techniques has helped improve the privacy preservation of users' data and extended the applications of machine learning models.
However, the involvement of a large number of users in FL also creates open opportunities for different adversaries, such as poisoning attacks, byzantine attacks, and adversarial examples attacks.
Yet, recent research has disclosed that existing poisoning attacks and byzantine attacks can not achieve satisfactory penetration in realistic FL scenarios caused by strong assumptions, \textit{e.g.,} client selection rate, and the ratio of malicious attackers. 
In this paper, the transferability of adversarial examples among different client models is analyzed to understand the relation between adversarial examples and clients' data distribution.
Moreover, to mitigate the attacks of transferable adversarial examples, we design a defense mechanism stemming from the transferability of model robustness by adversarial training.
As a result, through theoretical analysis of transferability, we gain insights into adversarial examples and the vulnerability of federated learning systems. 
Our proposed adversarial attack and defense methods are evaluated via real-life datasets in various settings to show their performance over the existing state-of-the-art methods.
\end{abstract}


\section{Introduction}
\label{sec:intro}
Federated Learning (FL)~\cite{mcmahan2017communication} has emerged as an attractive distributed machine learning framework, where multiple private data owners (\textit{e.g.,} clients) collaboratively train a machine learning model on the central organizer (\textit{i.e.,} the service provider) without sharing their private data.
In particular, each participating client performs a local update privately and sends the updated parameter to the service provider for aggregation in an iterative manner. 
Benefited from these features, FL has been adopted by a suite of prevalent applications at scale in the real world, such as the next word prediction model from Google G-board~\cite{konevcny2016federated}, the Siri voice recognition model from Apple~\cite{paulik2021federated}, and the credit information analysis model from different financial institutions~\cite{cheng2020federated}.  
To handle the diverse real-world applications, despite the existence of many variants of FL models, the original version, known as Federated Averaging (FedAvg)-based~\cite{mcmahan2017communication} methods, remains the most widely adopted paradigm due to its concise system design and outstanding scalability.
However, ``FedAvg'' does not fully understand the underlying challenges of system heterogeneity and statistical heterogeneity in diverse environments, where the clients possess different computational power and highly skewed non-identically distributed data in practice.
An advanced personalized solution ``FedProx''~\cite{li2020federated}, extended from FedAvg, is designed to hold the merits of FedAvg but also address the challenges of systems and data heterogeneity.
In both algorithms, the primary focus is to improve the model performance in FL at scale, but they overlook the security issues behind the scenes.   

Many potential attacks may harm the system security in FL since the collaborating participants are mutually untrusted clients, \textit{e.g.,} Android/iPhone users, and different organizations.
Therefore, any client can perform malicious operations or act as an honest-but-curious client even if they follow common system requirements.
As a result, the adversarial attacks\footnote{In this paper, we mainly investigate the adversarial attacks. We acknowledge that privacy inference is another important issue in FL, but it is beyond the scope of this work.} 
in FL can be categorized into three branches:
(i) Poisoning attacks (including backdoor attacks)~\cite{fang2020local,baruch2019little,sun2022data} that aim at causing degraded model performance on poisoned models or poisoned data;
(ii) Byzantine attacks~\cite{mean, bulyan} that allure the FL training converges to `sub-optimal to utterly ineffective models';
and (iii) Evasion attacks~\cite{wang2023potent,xiong2023exact, kim2023pfeddef} where the adversary tries to invade the model's integrity with an adversarially crafted input example.
Existing literature has intensively studied various methods in (i) and (ii) regarding these attacks and their corresponding defenses, while \textbf{they make some unrealistic assumptions that are hard to meet in real-world FL applications}.
For instance, in the related works on poisoning attacks~\cite{baruch2019little, fang2020local, shejwalkar2021manipulating} and defense~\cite{chen2018draco, xiong2024appro, yin2018byzantine}, authors assume that the adversarial party has the ability to compromise at most $10\%$ to $25\%$ of all FL clients, which means the attacker may need to control over 10 million users in a real-world App like Google G-board with billions of user base~\cite{bonawitz2019towards}.
On the other hand, for Byzantine attacks and defenses, different assumptions such as i.i.d. data in FL system~\cite{blanchard2017machine, xia2019faba}, a clean dataset on server~\cite{cao2020fltrust, xie2019zeno}, and the total number of FL clients~\cite{sattler2020byzantine, pillutla2022robust} are required to launch an attack or defense, which is hardly achievable in real applications.
Besides, in the recent paper by Shejwalkar \textit{et al.}~\cite{shejwalkar2022back}, where authors conduct a comprehensive review of existing attacks on FL through analysis and evaluation, \textbf{the conclusion is that most poisoning attacks and Byzantine attacks fail to achieve the claimed attack performance in FL at scale settings even in the absence of any robust defense mechanisms.} 

These unrealistic assumptions and experimental findings therefore motivate us to explore a more operable attack -- evasion attacks, with minimal attacker knowledge and a practical threat model.
In this paper, we revisit evasion attacks in the realistic FL setting, providing findings on theoretical attack performance and robust defense against the attack.
Specifically, considering both applicability and feasibility in real-world applications, ``FedProx'' paradigm is selected as the representative base case for analysis.
Built on which, exploring and rethinking the relationship between adversarial examples and their transferability in the ``FedProx'' FL system is our main target.
Our theoretical result is promising, where we find that the transferability of adversarial examples in the FL system can be harnessed by the local private data distribution of the malicious client. 
In addition to that, an efficient and robust defense mechanism is proposed to defend against adversarial examples in the FL system.

In summary, this work makes the following contributions:
\begin{itemize}
	\item We study the evasion attacks in FL at a realistic scenario, which is the first work to analyze the transferability of adversarial examples in FL through a theoretical manner, to the best of our knowledge.
	\item We design a robust defense framework in federated learning against adversarial example attacks, which has both efficiency and superior performance.
	\item Extensive experiments are conducted on popular FL datasets, reflecting our theoretical analysis on the transferability of adversarial examples and justifying the defense performance of our robust defense method.
\end{itemize} 

The remainder of this paper is organized as follows.
The related works in adversarial attacks and defenses are introduced in Section~\ref{sec:related_works}.
Then, the system settings and threat model are formulated in Section~\ref{sec:system_model}, followed by the analysis of attack transferability in Section~\ref{sec:attack}.
Next, the robust defense mechanism is proposed in Section~\ref{sec:methodology}, and extensive evaluations are conducted in attack and defense scenarios in Section~\ref{sec:exp}.
Finally, the paper is concluded in Section~\ref{sec:con}

\section{Related Works}
\label{sec:related_works}

\subsection{Adversarial Attacks in FL}
In FL, adversarial attacks can stem from clients, the central server, or the communication channel. 
This paper focuses on the client as the attack source.
As insiders participating in the training phase, clients can launch causative attacks affecting model performance during both training and inference stages. 
In the training stage, poisoning and backdoor attacks can be launched by a malicious participant through submitting poisonous data or tampering with model updates during training that skew the model’s predictions~\cite{fang2020local,baruch2019little,sun2022data}.
On the contrary, evasion attacks~\cite{wang2023potent,kim2023pfeddef}, can manipulate the model’s predictions by crafting an adversarial input as test data during inference.
The target of adversarial attacks can also be the clients' privacy information, where attackers can use learning based methods to reconstruct private client data using gradient-based data or extract label information~\cite{hitaj2017deep,fu2022label}.

The impact of adversarial attacks in FL may vary. 
For utility-centric attacks, it is measured as the reduction in the test accuracy of the model after the attack, usually quantified by the accuracy drop as the attack success rate.
For privacy-centric attacks, attackers aim to perform membership/label inference or model inversion, and the impact can have different metrics like mean squared error (MSE), peak signal-to-noise ratio (PSNR), structural similarity index (SSIM), and learned perceptual image patch similarity (LPIPS). 
In this work, our main focus is on the evasion attack in FL with theoretical analysis, which has not been investigated in previous works, to the best of our knowledge.

\subsection{Robust Defenses in FL}
In response to adversarial attacks in FL, various mechanisms have been proposed to build safe and robust FL systems.
In this section, we only list those representatives in enhancing robustness while omitting the privacy-preserving techniques, as they are related to our work.
PruneFL~\cite{pruningFL} and Network Pruning~\cite{netprune} are examples of pruning-based defense methods, where pruning algorithms can be applied during the aggregation or local update to remove malicious parameters generated by poisoning attacks.
Byzantine robust aggregation methods~\cite{krum, mean, bulyan} are typically implemented on the server side and filter out malicious updates and ensure that only valid updates are used in the model.
Regularization is also a defense strategy applicable to both server and client sides in FL, protecting against data poisoning and model poisoning attacks~\cite{frl}.
Specifically, on the client side, participants employ regularization techniques like dropout, batch normalization, or weight decay during their model training (LSR~\cite{LSR}, ConTre~\cite{ConTre}), which combats overfitting and contributes to overall FL robustness. 
These methods are mainly used to cure poisoning attacks, yet the investigated problem in this paper is evasion attacks.
So far, only adversarial training is an effective defense strategy against evasion, and the state-of-the-art is Federated Adversarial Training (FAT)~\cite{fat}.

However, the adversarial training framework can only work for known attacks presented in the adversarial training dataset, with a high computation cost.
Defending evasion attacks on the local client side with an efficient mechanism is still an open challenge, which motivates our matrix decomposition-based method in Section~\ref{sec:methodology}.

\section{System Setting \& Threat Model}
\label{sec:system_model}

\subsection{System Setting}
\label{sec:system}

Our system setting is adopted from~\cite{li2020federated}, in which authors devise ``FedProx'' that can handle the system heterogeneity (\textit{e.g.,} straggler caused by computing power and network connection) and statistical heterogeneity (\textit{e.g.,} different local optima caused by highly non-i.i.d. data).
Specifically, the system of FedProx consists of a central server and $K$ clients.
Each local client $k\in \{1,2,...,K\}$ holds a local dataset $D_k=\{(x_i,y_i)\}_{i=1}^{n_k}$ and performs local training algorithm to minimize the objective function $F_k(w;w^t)$, which is defined as follows.
\begin{align}
	\label{eq:local_obj}
	F_k(w;w^t)= \frac{1}{n_k} \sum_{i=1}^{n_k} [l(f(w;x_i),y_i)]+\frac{\mu}{2}\|w-w^t\|^2,
\end{align}
where the first term is the loss value of model parameter $w$ on the local dataset $D_k$, and $w^t$ in the second term is the global model parameter output by the $t$-th round server aggregation.
The optimized model parameter of client $k$ for this round is defined as $w_k^{t+1}:=\min_w F_k(w;w^t)$.
Following this, the updated aggregated model $w^{t+1}$ is constructed on the server through weighted averaging as shown in Eq.~\eqref{eq:average}
\begin{align}
	\label{eq:average}
	w^{t+1}=\sum_{k=1}^{K}p_k w_k^{t+1},
\end{align}
where $p_k\geq0$ and $\sum_{k} p_k=1$.
Generally, $p_k$ is set to be $p_k=\frac{n_k}{n}$, where $n=\sum_{k}n_k$ is the total number of data points in the FL system.
With the local objective function and the aggregation rule, the federated optimization objective of FedProx is formulated as Eq.~\eqref{eq:fedprox}
\begin{align}
	\label{eq:fedprox}
	\min_w F(w)=\mathbb{E}_k[F_k(w;w^t)]=\sum_{k=1}^{K}p_k F_k(w;w^t).
\end{align}

Regarding this system setting, we have made some commonly used assumptions to facilitate our analysis in later sections.

\begin{itemize}
	\item \textbf{(Lipschitz Continuity)} $\forall w, w' \in \mathcal{W}$, the loss function $l(f(w;x),y)$ is Lipschitz continuous with $\iota>0$~\cite{wei2020federated,hu2020differentially}:
	\begin{align}
		\|l(f(w;x),y)-l(f(w;x),y)\|\leq \iota \|w-w'\|.
	\end{align}

	\item \textbf{(Lipschitz Continuous Gradient)} $\forall w, w' \in \mathcal{W}$, the gradient of the loss function $l(f(w;x),y)$ is Lipschitz continuous with $\lambda>0$~\cite{wei2020federated, li2019convergence}:
	\begin{align}
		\|\nabla l(f(w;x),y) -\nabla l(f(w';x),y)\|\leq \lambda \|w-w'\|.
	\end{align}
\end{itemize}

\subsection{Threat Model}
\label{sec:threat_model}

In this work, the FedProx framework is considered as the cyber-range of a malicious client (\textit{referred to as ``the attacker'' hereafter}).
The \textbf{goal} of the attacker is to generate adversarial examples that can cause performance degradation (e.g., misclassification) on a targeted client (\textit{referred to as ``the victim'' hereafter}).

The attacker may be a current OR previous client in this FL system,  holding his/her local private dataset and local model parameters (i.e., model parameter $w_k^t$) as the FL training evolves.
In such a setting, the \textbf{knowledge} of the attacker is only the task of the FL system and his own data and model.
As we are in the FedProx setting, the victim's model parameters do not necessary to be the same as the attacker's, resulting in a black-box setting.

The \textbf{attack vector} of the attacker is described as follows:
(i) The attacker joined the FL system to receive the initialized model and task, and started the FL process. 
(ii) Based on the attacker's own local model, he can craft adversarial examples locally, which prevents possible detection mechanisms.
These well-functioning adversarial examples can attack the attacker's local model successfully.
(iii) Once the adversarial example is generated, the transferability of the adversarial example will be utilized as the key of an attacker to invade many victim clients with less effort.
The generated adversarial examples will be spread widely or distributed to the target victim.
When the victim uses them as inputs, the attack will be effective due to the transferability of adversarial examples among different machine learning models~\cite{guo2020backpropagating, goodfellow2014explaining}.
This attack can be implemented even when there is only one attacker, regardless of the total number of clients or the model parameter/structure of the victim client.

Indeed, the transferability of adversarial examples is mostly studied via experimental results in previous literature, and the cause of transferability is still an open question.
In this threat model, we want to not only investigate the adversarial example attack in FL at scale but also analyze the key factors that may influence the transferability of attack examples to inspire further research on attack and defense.

\section{Transferable Adversarial Attack}
\label{sec:attack}

\subsection{Adversarial Example Attack in FL}
Without loss of generality, we consider the most common $C$ classification task in FL, where the clients' models are separately trained on their private dataset $D_k=\{(x_i,y_i)\}_{i=1}^{n_k}$ and $x_i \in \mathcal{X}$ is the feature space of data while $y_i\in \mathcal{Y}=\{1,2,\dots, C\}$ is the label space.
The model of each client is denoted as a function $f(w;x): \mathcal{X} \rightarrow \mathbb{R}^{|\mathcal{Y}|}$ mapping the input data to a probability simplex, where the $c$-th element $f_c(w;x)$ in the output $f(w;x)$ represents the probability of $x$ belongs to $c$-th class.
Therefore, $\arg \max_c f_c(w;x)$ is the classified label computed by model $w$.
Also, there is a continuous loss function (usually cross-entropy loss) $l(f(w;x),y): \mathcal{X}\times \mathcal{Y} \rightarrow \mathbb{R} $ associated with each classifier, where $y$ is the ground-truth label of corresponding data $x$.
Generally, a correct classification result will have the smallest loss value, while a misclassified label can induce a larger loss.

For an attacker $A$ with classifier function $f(w_A;\cdot)$, when an input data $x$ is correctly predicted, the ground-truth label $y = \arg \max_c f_c(w;x)$.
To achieve the desired [untarget adversarial attack], an adversarial example $x'$ can be generated through different attack methods so that the crafted adversarial example $x'$ is misclassified into a different label $y'\neq y$.
Under this situation, the loss function value of $x'$ with respect to $w_A$ will increase, which means that $l(f(w_A;x'),y)$ becomes higher than the original loss $l(f(w_A;x),y)$ on clean data $x$.
In practice, maximizing the loss function value based on model parameters is the most common way of constructing adversarial examples in existing literature, such as FGSM~\cite{goodfellow2014explaining}, PGD~\cite{madry2017towards}, C\&W~\cite{carlini2017towards}, etc.
Moreover, from~\cite{szegedy2013intriguing}, we know that the imperceptible adversarial examples can sabotage other unseen models with different hyper-parameters (cross-model generalization) or trained on disjoint sub-datasets (cross-training data generalization).
This transferability of $x'$ is able to increase the loss function value of other clients in the FL system with different model parameters other than $w_A$.
In this regard, the increased loss value is very likely to push $x'$ across the decision boundary to make a misclassification result. 
Therefore, the adversarial example generated by the attacker $A$ can penetrate other benign clients even if they are personalized with non-i.i.d data in the federated setting (\textit{e.g.,} our target FedProx). 

For an individual victim client $V$ with local model parameter $w_V$, when the adversarial example $x'$ generated by the attacker $A$ is taken as an input, the classification function output $f(w_V;x')$ and loss function value $l(f(w_V;x'),y)$ will shift away from the normal level.
Since the loss functions (\textit{e.g., }cross-entropy) are always continuous and differentiable, the difference between loss function value $l(f(w_V;x'),y)$ on the victim and $l(f(w_A;x'),y)$ on the attacker can be quantified by analyzing model parameters $w_V$ and $w_A$.
If the adversarial example $x'$ works on model $w_A$ with a large loss function value, the loss value $l(f(w_V;x'),y)$ should also be large, which means the attack is transferable from the attacker to victim's model $w_V$.
Our purpose of this work is to investigate the relation between $l(f(w_V;x'),y)$ and $l(f(w_A;x'),y)$, and therefore to understand the transferability of adversarial example as well as its associated factors.

\textbf{Remark.} 
There are a wealth of attack methods~\cite{goodfellow2014explaining, madry2017towards, carlini2017towards, su2019one, xie2019improving} available to generate this adversarial example $x'$ yet they are out of the scope of our work.
In this study, we primarily focus on the transferability of this adversarial example regardless of what methods are used to generate it.

\subsection{Transferability of Adversarial Example}

In this work, we define the transferability of adversarial examples from the perspective of loss function value similarity between the attacker and the victim.
\begin{definition}
	(Transferability of Adversarial Example). 
	Given an adversarial example $x'$, a source model parameter $w_A$ used to generate $x'$, and a target model parameter $w_V$ to be attacked, the difference between loss values, $\|l(f(w_A;x'),y) - l(f(w_V;x'),y)\|$, measures the transferability of the adversarial example $x'$. 
\end{definition}
Intuitively, the smaller the distance between the two losses, the higher the chance that $x'$ can invade model $w_V$, which expresses a stronger transferability.
Next, we dive in to seek what is the underlying factors influencing the transferability of an adversarial example among different clients in the FedProx system and formulate the findings within Theorem~\ref{tm:transfer}.

\begin{theorem}
	\label{tm:transfer}
	The transferability of adversarial example $x'$ between an attacker $A$ and any victim client $V$ is related to local data distribution and given in the following equation.
    {\small
	\begin{align}
		\notag	&\|l(f(w_A;x'),y) - l(f(w_V;x'),y)\| \leq \\
        &\iota \eta  \|\sum_{c=1}^{C} p_V(y_i=c) - p_A(y_i=c)\|+ \sum_{l=0}^{m-1}((b_V)^{l} G_{max}(w_A^{mt-1-l})),
	\end{align}}
	where $p_A$ and $p_V$ are the local data distribution of the attacker and the victim, $G_{max}(\cdot)$ is the maximal value of gradient, $b_V=(1-\eta \mu + \eta \sum_{c=1}^{C} p_V(y_i=c) \gamma)$, and some constant $\iota$, $\eta$, $\gamma$ are given in common assumptions.  
\end{theorem}

This theorem reveals that the transferability of an adversarial example is mostly related to the data distribution difference between its source client (attacker) and the target client (victim).

\begin{proof}
	To accomplish this proof, some commonly used assumptions in optimization theory are needed, for more details, please refer to Section~\ref{sec:system}.
	
	According to the Lipschitz Continuity, we can derive 
	\begin{align}
		\label{eq:tran_proof}
		\|l(f(w_A;x'),y) - l(f(w_V;x'),y)\| \leq \iota \|w_A - w_V\|.
	\end{align}
	Therefore, finding the upper bound $\|l(f(w_A;x'),y) - l(f(w_V;x'),y)\|$ (\textit{i.e.,} the transferability between $w_A$ and $w_V$) is transformed to the upper bound of $\|w_A - w_V\|$.
	This distance $\|w_A - w_V\|$ between the attacker and victim may be changed in each round, and in most federated learning, including our target FedProx, the local clients train their local models (\textit{e.g.,} $w_A$ and $w_V$) separately in an iterative manner.
	Thus, without loss of generality, we consider a universal case where the global training has $T$ rounds and in each round the local client performs $m$ times iterative updates, which means a local client will have $mT$ times updates.
	
	For a local client $k$, according to the objective function defined in Eq.~\eqref{eq:local_obj}, the cross-entropy loss of the $m$-th update during the $t$-th round can be rewritten in the following format,
	\begin{align}
		\notag&F_k(w_k^{mt}; w^{t-1})\\ \notag
		=& \frac{1}{n_k} \sum_{i=1}^{n_k} [l(f(w_k^{mt};x_i),y_i)]+\frac{\mu}{2}\|w_k^{mt}-w^{t-1}\|^2, \\\notag
		=& \mathbb{E}_{(x_i,y_i) \in D_k}[l(f(w_k^{mt};x_i),y_i)]+\frac{\mu}{2}\|w_k^{mt}-w^{t-1}\|^2, \\\notag
		=& \mathbb{E}_{(x_i,y_i) \in D_k}[\sum_{c=1}^{C} \textcolor{black}{y_{c}} \log f_c(w_k^{mt};x_i)]+\frac{\mu}{2}\|w_k^{mt}-w^{t-1}\|^2,\\
		=& \sum_{c=1}^{C} p_k(y_i=c) \mathbb{E}_{x_i|y_i=c}[\log f_c(w_k^{mt};x_i)]+\frac{\mu}{2}\|w_k^{mt}-w^{t-1}\|^2.
	\end{align}
	where $\textcolor{black}{y_{c}}$ is the one-hot vector with 1 in the $c$-th position, and $p_k(y_i=c)$ is the data distribution in dataset $D_k$. 
	Based on this reformulated loss function, in the next iteration, a new local model $w_k^{mt+1}$ is updated as follows,
	\begin{align}
		w_k^{mt+1} = &w_k^{mt} - \eta \nabla F_k(w_k^{mt}; w^{t-1})
	\end{align} 
	where $\eta$ is the learning rate, and $\nabla F_k(w_k^{mt}; w^{t-1})$ is the gradient with respect to model parameters.
	We can calculate the gradient $\nabla F_k(w_k^{mt}; w^{t-1})$ in Eq.~\eqref{eq:local_g}
	\begin{align}
		\label{eq:local_g}
		\notag \nabla F_k(w_k^{mt}; w^{t-1}) =& \sum_{c=1}^{C} p_k(y_i=c) \nabla_w \mathbb{E}_{x_i|y_i=c}[\log f_c(w_k^{mt};x_i)] \\ &+\mu (w_k^{mt}-w^{t-1}).
	\end{align}
	For simplicity, hereafter we use $G(w_k^{mt})$ to represent the gradient term $\nabla_w \mathbb{E}_{x_i|y_i=c}[\log f_c(w_k^{mt};x_i)]$, and Eq.~\eqref{eq:local_g} is simplified to 
	\begin{align}
		\nabla F_k(w_k^{mt}; w^{t-1}) = \sum_{c=1}^{C} p_k(y_i=c)G(w_k^{mt}) +\mu (w_k^{mt}-w^{t-1}).
	\end{align}
	
	Now, we can compute the distance between the attacker's model $w_A$ and the victim's model $w_V$.
	In the $m$-th iteration of the $t$-th round, the distance is calculated as follows,
	{\small
    \begin{align}
		\label{eq:av_dis_1}
	\notag	& \|w_A^{mt} - w_V^{mt}\| = 
 \|(w_A^{mt-1} - \eta (\textcolor{black}{(\nabla F_A(w_A^{mt}; w^{t-1})})\\
 \notag &-  (w_V^{mt-1} - \eta (\textcolor{black}{(\nabla F_V(w_V^{mt}; w^{t-1})}) \| \\
	\notag =& \|w_A^{mt-1} - w_V^{mt-1} -\eta \sum_{c=1}^{C} p_A(y_i=c)G(w_A^{mt-1}) \\
	\notag		& + \sum_{c=1}^{C} p_V(y_i=c)G(w_V^{mt-1}) - \eta \mu (w_A^{mt-1} - w_V^{mt-1}) \| \\
	\notag =& \|(1-\eta \mu)(w_A^{mt-1} - w_V^{mt-1}) \\ 
	\notag	+& \eta (\sum_{c=1}^{C} p_V(y_i=c)G(w_V^{mt-1}) - \sum_{c=1}^{C} p_A(y_i=c)G(w_A^{mt-1}))  \| \\
	\notag \leq& \|(1-\eta \mu)(w_V^{mt-1} - w_A^{mt-1})\| \\
		    +& \eta \| (\sum_{c=1}^{C} p_V(y_i=c)G(w_V^{mt-1}) - \sum_{c=1}^{C} p_A(y_i=c)G(w_A^{mt-1})) \| 
	\end{align}}
	
The second term in Eq.~\eqref{eq:av_dis_1} can be rewritten as follows,
{\small
	\begin{align}
    \label{eq:dist_2}
	\notag	&\eta \| (\sum_{c=1}^{C} p_V(y_i=c)G(w_V^{mt-1})) - (\sum_{c=1}^{C} p_A(y_i=c)G(w_A^{mt-1})) \| \\ 
	\notag = & \eta \|(\sum_{c=1}^{C} p_V(y_i=c)G(w_V^{mt-1})) - (\sum_{c=1}^{C} p_V(y_i=c)G(w_A^{mt-1})) \\
	\notag  &+(\sum_{c=1}^{C} p_V(y_i=c)G(w_A^{mt-1})) - (\sum_{c=1}^{C} p_A(y_i=c)G(w_A^{mt-1})) \| \\
	\notag \leq& \eta \sum_{c=1}^{C} p_V(y_i=c) \|G(w_V^{mt-1}) - G(w_A^{mt-1}) \|	\\
	\notag  &+ \eta G(w_A^{mt-1}) \|\sum_{c=1}^{C} p_V(y_i=c) - p_A(y_i=c)\| \\
	\notag \overset{(i)}{\leq}& \eta \sum_{c=1}^{C} p_V(y_i=c) \gamma \|w_V^{mt-1} - w_A^{mt-1} \| \\
	\notag  &+ \eta G_{max}(w_A^{mt-1}) \|\sum_{c=1}^{C} p_V(y_i=c) - p_A(y_i=c)\| \\
	\end{align}}%
	The inequality ($i$) holds because of the Lipschitz Continuous Gradient condition of the loss function, and $G_{max}(w_A^{mt-1})$ is the maximal value of the gradient in $w_A^{mt-1}$.
	
	Therefore, combining the inequality Eq.~\eqref{eq:dist_2} and Eq.~\eqref{eq:av_dis_1}, we can have a one step upper bound for $\|w_A^{mt} - w_V^{mt}\|$ represented by $w_A^{mt-1}$ and $w_V^{mt-1}$ as shown in Eq.~\eqref{eq:av_dis_t_1}.
	\begin{align}
		\label{eq:av_dis_t_1}
		\notag	& \|w_A^{mt} - w_V^{mt}\| \\
		\notag \leq& (1-\eta \mu + \eta \sum_{c=1}^{C} p_V(y_i=c) \gamma) \|(w_V^{mt-1} - w_A^{mt-1})\| \\
		  &+ \eta G_{max}(w_A^{mt-1}) \|\sum_{c=1}^{C} p_V(y_i=c) - p_A(y_i=c)\| 
	\end{align}
	
	Based on this inequality, we can reduce to the general case of the distance between $w_A^{mt} $ and $ w_V^{mt}$.
	To ease the reading experience, we use $b_V$ to denote $(1-\eta \mu + \eta \sum_{c=1}^{C} p_V(y_i=c) \gamma)$ for simplicity.
	Then we will have the following induction.
    \begin{align}
		\label{eq:av_distance}
		\notag	& \|w_A^{mt} - w_V^{mt}\| \\
		\notag \leq& b_V \|(w_V^{mt-1} - w_A^{mt-1})\| + \\
        \notag & \eta G_{max}(w_A^{mt-1}) \|\sum_{c=1}^{C} p_V(y_i=c) - p_A(y_i=c)\| \\
		\notag \leq& (b_V)^2 \|(w_V^{mt-2} - w_A^{mt-2})\| + \\
        \notag & \eta G_{max}(w_A^{mt-1}) \|\sum_{c=1}^{C} p_V(y_i=c) - p_A(y_i=c)\| \\
		\notag &+ \eta b_V G_{max}(w_A^{mt-1}) \|\sum_{c=1}^{C} p_V(y_i=c) - p_A(y_i=c)\| \\
		\notag =& (b_V)^2 \|(w_V^{mt-2} - w_A^{mt-2})\| + \eta  \|\sum_{c=1}^{C} p_V(y_i=c) - p_A(y_i=c)\| \\ 
       \notag & \times (G_{max}(w_A^{mt-1}) + b_V G_{max}(w_A^{mt-2}) ) \\ 
		\notag & \dots \\
		\notag \leq& (b_V)^{m} \|(w_V^{mt-m} - w_A^{mt-m})\| + \\
		\notag & \eta  \|\sum_{c=1}^{C} p_V(y_i=c) - p_A(y_i=c)\| \sum_{l=0}^{m-1}((b_v)^{l} G_{max}(w_A^{mt-1-l})) \\
		\overset{(i)}{\leq}& \eta  \|\sum_{c=1}^{C} p_V(y_i=c) - p_A(y_i=c)\| \sum_{l=0}^{m-1}((b_v)^{l} G_{max}(w_A^{mt-1-l})) 
	\end{align}
	This final inequality ($i$) holds because the model parameters $w_V^{mt-m}:=w^{t-1}$ is the aggregated global parameter in the previous round (\textit{i.e.,} the ($t-1$)-th round) and $w_A^{mt-m}:=w^{t-1}$ is also the aggregated model parameter by definition.
	Thus, $ \|(w_V^{mt-m} - w_A^{mt-m})\|$ equals to $0$.
	
	With the upper bound of model parameters distance $\|w_A^{mt} - w_V^{mt}\|$, the transferability of adversarial example can be calculated by substituting Eq.~\eqref{eq:av_distance} into Eq.~\eqref{eq:tran_proof}.
	So, we can derive the transferability as follows
	\begin{align}
    \label{eq:final}
	\notag	&\|l(f(w_A;x'),y) - l(f(w_V;x'),y)\| \leq\\ 
		\iota \eta  & \|\sum_{c=1}^{C} p_V(y_i=c) - p_A(y_i=c)\|+ \sum_{l=0}^{m-1}((b_V)^{l} G_{max}(w_A^{mt-1-l})).
	\end{align}
	This finishes the proof of Theorem~\ref{tm:transfer}.
\end{proof}

This theorem confirms that there is a positive relation between the transferability of adversarial examples and the distribution differences between the victim and the attacker client.
That is, when the data distribution of the attacker is similar to that of the victim, the crafted adversarial examples based on the source model of the attacker will be more transferable and form a successful attack.
We conduct a series of experiments on the transferability of adversarial examples to catch the relation in Section~\ref{sec:exp}, and the experimental results are in compliance with our analysis in Theorem~\ref{tm:transfer}.

\section{Methodology: Robust Defense in FL}
\label{sec:methodology}

\begin{figure}[tbp]
    \centering
    \includegraphics[width=\linewidth]{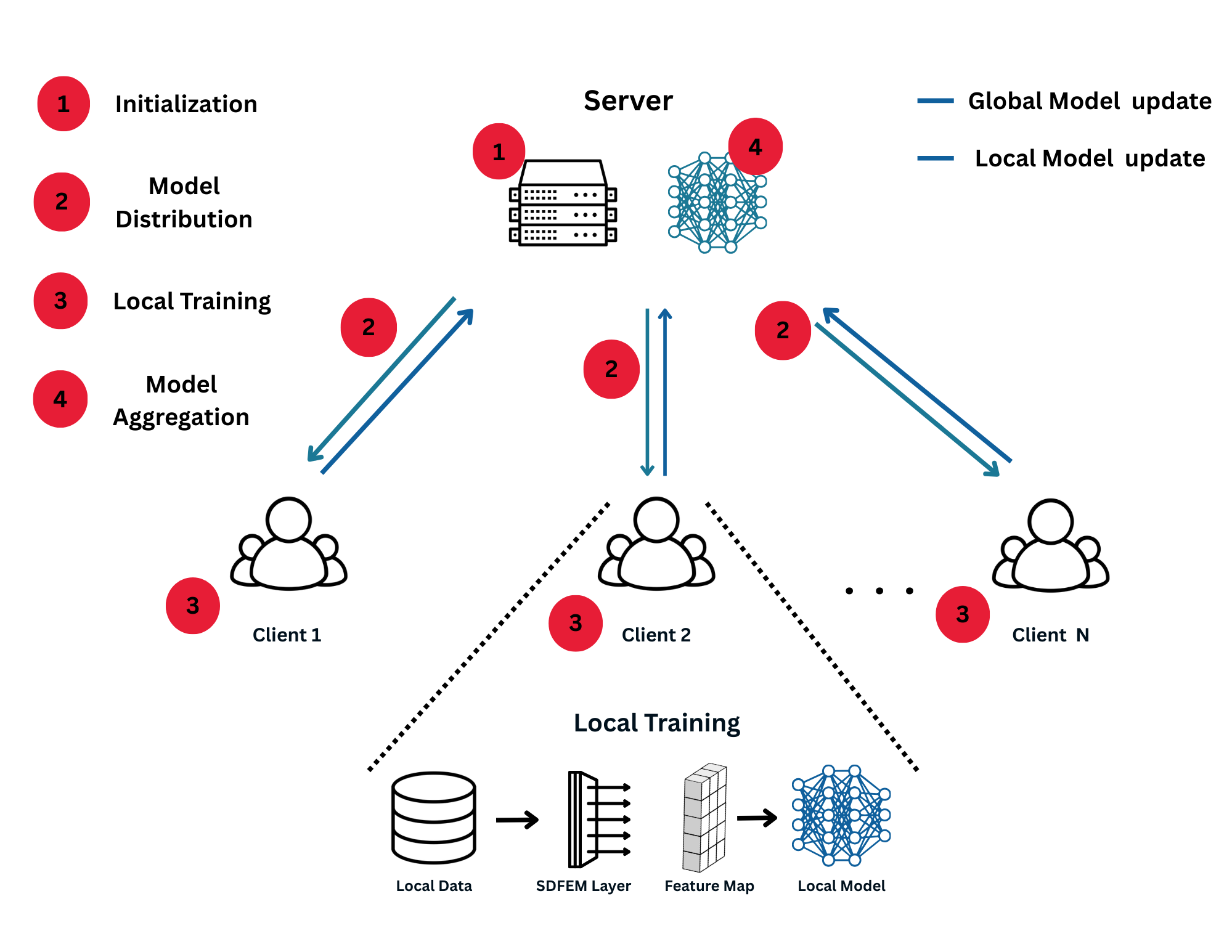}
    \caption{Overview of the Federated Learning Framework used in this study.}
    \label{fig:placeholder}
\end{figure}


Motivated by our theoretical analysis in Section ~\ref{sec:attack}, which reveals that the transferability of adversarial examples stems from similarities from the data-distribution and the weights of the attacker model and the victim model. So, to mitigate the transferability of the samples from attacker we need to make the victim model learn more robust and unique features. 
To employ our solution we used SVD-driven Feature Evolution Module (SDFEM) and Adversarial Training alongside it to force the model to learn robust features. It gave great results in centralized settings by Liu et al.~\cite{Liu2025CVPR}.

\subsection{SDFEM for Feature Robustness}
Inspired by Liu et al.~\cite{Liu2025CVPR}, the Singular Value Decomposition (SVD) method can be applied to the input of models.  
Its primary goal is to suppress feature channels that are highly sensitive to adversarial perturbations while retaining robust semantic features.
The SDFEM module is deployed at the input layer of the client's local model without changing the original model parameters, which can serve as a plug-in kit.

Let $x \in \mathbb{R}^{h \times w}$ represent an input image sample. 
The module first decomposes the feature map using SVD:
\begin{equation}
 x = U \Sigma V^T,
\end{equation}
where $U$ and $V$ are orthogonal matrices, and $\Sigma$ is a diagonal matrix containing the singular values $\sigma_i$ which represent the energy of feature channels.

\begin{algorithm}[t]
\caption{Robust FedProx with SDFEM}
\label{alg:robust_fed}
\SetAlgoLined
\textbf{Server executes:}\\
Initialize $w^0$, global mask $M$ \\
\For{round $t = 0, 1, \dots, T-1$}{
    Select subset of clients $S_t \subseteq \{1, \dots, K\}$\\
    Broadcast $w^t$ to all $k \in S_t$\\
    \For{each client $k \in S_t$}{
        $w_k^{t+1}, M_k^{t+1} \leftarrow \text{ClientUpdate}(k, w^t)$
    }
    $w^{t+1} \leftarrow \sum_{k \in S_t} p_k w_k^{t+1}$
}
\textbf{ClientUpdate($k, w^t$):}\\
Generate $D_k^{adv}$ using local surrogate\\
Initialize $w_k \leftarrow w^t$\\
\For{local epoch $e = 1 \dots E$}{
    Update $w_k, M_k$ to minimize Eq.~\eqref{eq:robust_local_obj}
}
\Return $w_k, M_k$
\end{algorithm}

To distinguish between robust and non-robust features, we introduce a learnable evolution mask $M \in \mathbb{R}^{h \times w}$. The input is reconstructed by re-weighting the singular values via element-wise multiplication with this mask:
\begin{equation}
    \text{SDFEM}(x, M) = U (\Sigma \odot M) V^T,
\end{equation}
where $\odot$ denotes the Hadamard product. 
By learning $M$ during training, the model can automatically attenuate singular values that contribute to the gradients exploited by transfer-based attacks, effectively lowering the gradient maximum term $G_{max}$ discussed in Theorem~\ref{tm:transfer}.

\subsection{Adversarial Training in Federated Aggregation}
The SDFEM method can only protect the local model, and generate robust local features and parameters.
To ensure the global model aggregates robust features, the federated learning optimization problem needs to be modified as well.
Adversarial training is adopted in the robust federated aggregation.
The local models from SDFEM can propagate their robustness to other clients and learn from each other to increase the global robustness against unseen data~\cite{xiong2021privacy}.
The major steps for adversarial training in FedProx are as follows.

\textbf{Adversarial Data Generation.}
In each round, client $k$ utilizes a local surrogate model to generate a batch of adversarial examples $D_k^{adv} = \{(x'_i, y_i)\}$ from the local clean data $D_k$. 
These examples are generated using PGD to maximize the local loss, serving as proxies for transferable attacks.

\begin{figure*}[htbp]
	\centering
	\subfloat[SVHN ASR]{\includegraphics[width=0.25\linewidth]{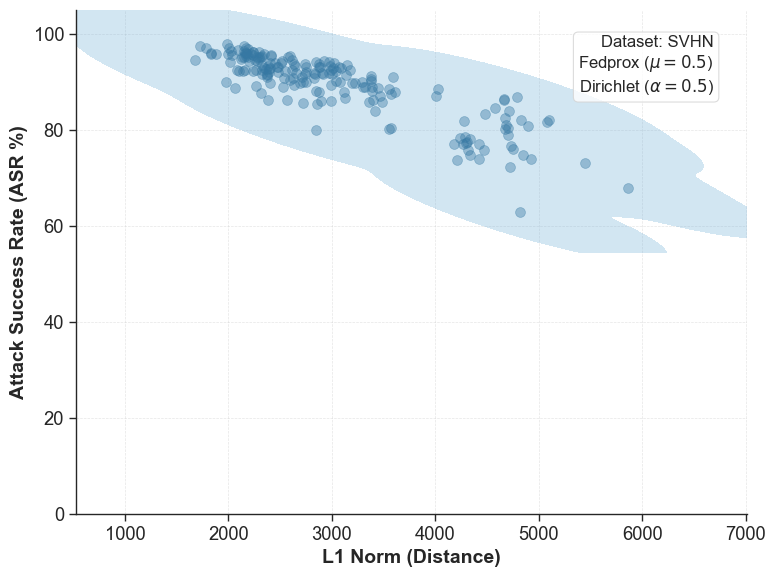}\label{fig:svhn_fedprox}}
	\hfil
	\subfloat[Cifar10 ASR]{\includegraphics[width=0.25\linewidth]{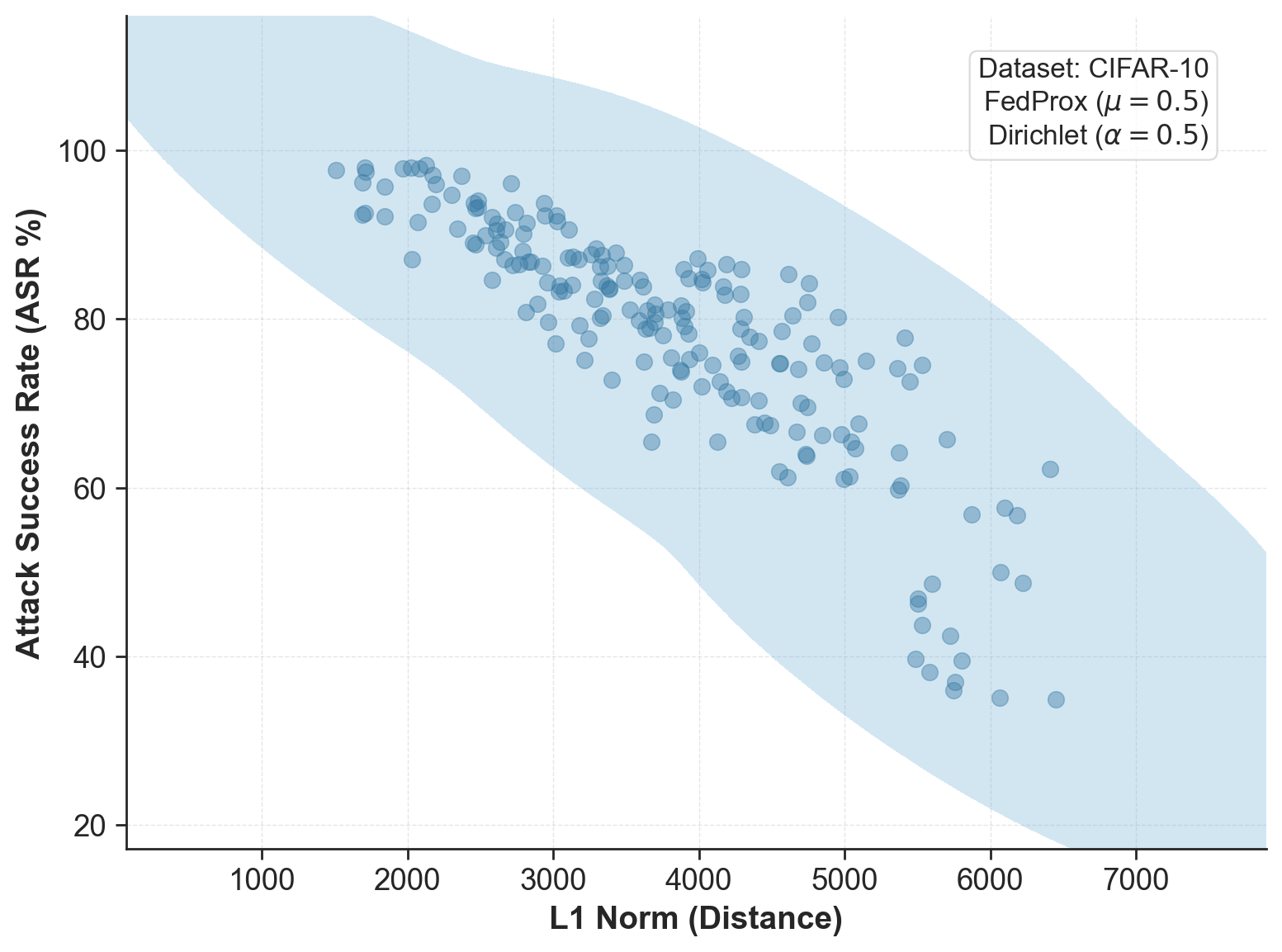}\label{fig:cifar10_fedprox}}
	\hfil
	\subfloat[SVHN Transferability]{\includegraphics[width=0.25\linewidth]{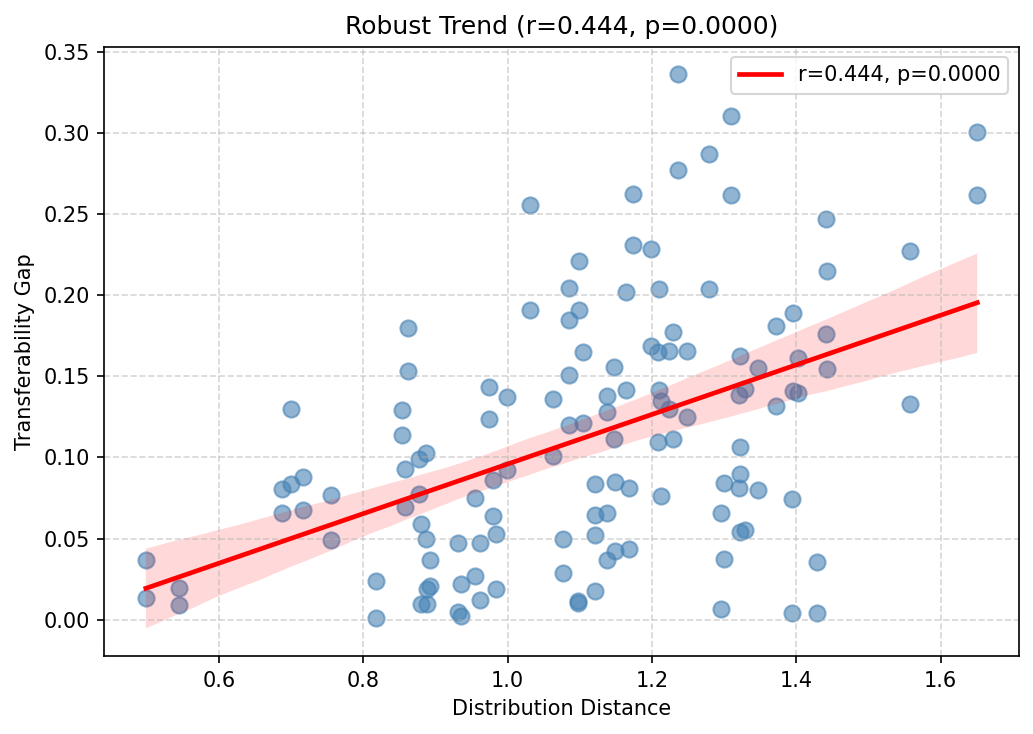}\label{fig:svhn_transfer}}
	\hfil
	\subfloat[Cifar10 Transferability]{\includegraphics[width=0.25\linewidth]{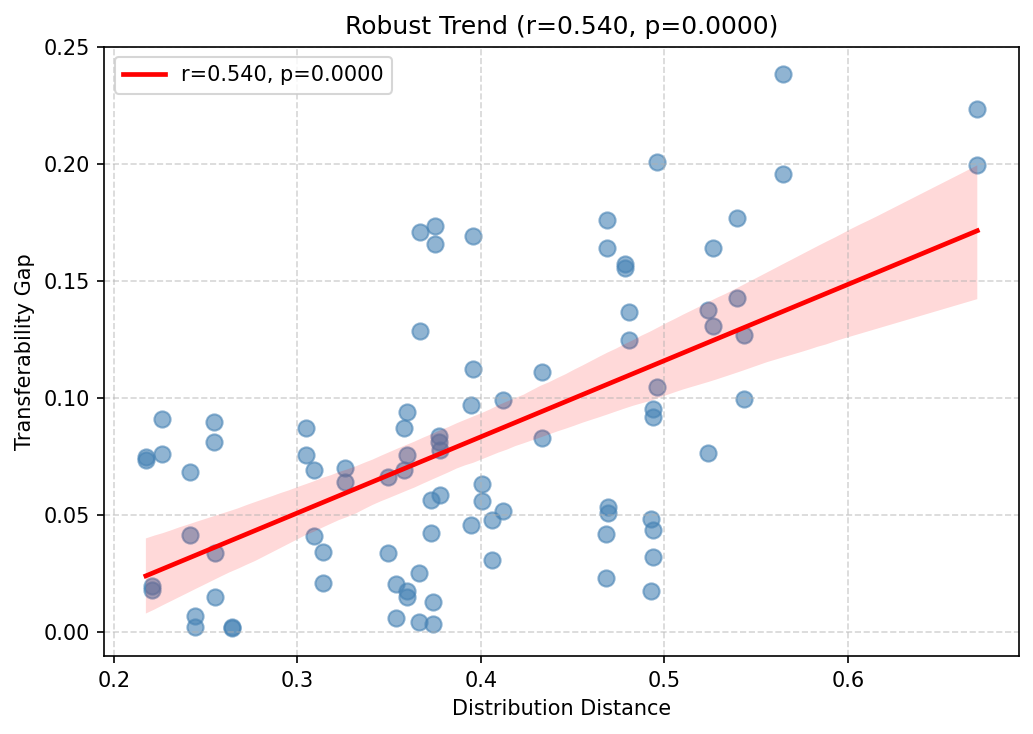}\label{fig:cifar10_transfer}}
	\caption{Correlation between Model Distance and Attack Success Rate (Fig.~\ref{fig:svhn_fedprox},~\ref{fig:cifar10_fedprox}), and between Distribution Distance and Attack Transferability (Fig.~\ref{fig:svhn_transfer},~\ref{fig:cifar10_transfer}) across different datasets.}
	\label{fig:correlation_all}
\end{figure*}

\textbf{Modified Local Objective.}
The local objective function $F_k(w; w^t)$ is updated to jointly optimize the model weights $w_k$ and the feature evolution mask $M_k$. 
The robust client minimizes the objective function in Eq.~\eqref{eq:robust_local_obj}.
\begin{align}
    \label{eq:robust_local_obj}
    \min_{w_k, M_k} \quad & \alpha \mathbb{E}_{(x,y) \in D_k} [l(f(w_k; \text{SDFEM}(x, M_k)), y)] \nonumber \\
    + & (1-\alpha) \mathbb{E}_{(x',y) \in D_k^{adv}} [l(f(w_k; \text{SDFEM}(x', M_k)), y)] \nonumber \\
    + & \frac{\mu}{2} \|w_k - w^t\|^2,
\end{align}
where $\alpha$ is a hyperparameter balancing clean accuracy and robustness. 
The full training procedure for our proposed method is in Algorithm~\ref{alg:robust_fed}.

\section{Experiments}
\label{sec:exp}

In this section, we aim to answer the following questions via experiments: 
(1) Do our experimental findings align with the theoretical correlation between data distribution and transferability derived in Theorem \ref{tm:transfer}? 
(2) Does the proposed robust mechanism effectively mitigate transferable adversarial attacks in a federated setting?

\subsection{Experimental Setup}

\subsubsection{Datasets and Models}
We utilize two benchmark image classification datasets: \textbf{CIFAR-10} and \textbf{SVHN}. 
For the model architectures, we employ \textbf{VGG11} and WideResNet-28-10 \textbf{(WRN28)} to represent standard and high-capacity deep learning models, respectively.

\subsubsection{Federated Learning Settings}
We simulate a federated learning environment with $K\in \{10,100\}$ clients. 
In each communication round, a fraction $C=\{0.8,1\}$ of clients are randomly selected to participate.
To simulate realistic statistical heterogeneity, we partition the data among clients using two settings:
\begin{itemize}
\item \textbf{IID Setting:} Data is shuffled and uniformly distributed across all clients.
\item \textbf{Non-IID Setting:} We sample class labels from a Dirichlet distribution with concentration parameter $\alpha\in\{0.1,0.5,1\}$, where a smaller $\alpha$ indicates higher data skewness.
\end{itemize}
The global model is trained for $T=400$ rounds. For local updates, we use an SGD optimizer with a learning rate of $\eta=0.01$, momentum $0.9$, and local epoch $E=1$ for ideal IID settings and $E=3$ for Non-IID settings.

\subsubsection{Attack and Defense Implementation}
For the attack generation, we assume a gray-box threat model where the adversary uses their local surrogate model to generate transfer-based attacks. 
We evaluate against four state-of-the-art gradient-based attacks: \textbf{VMIFGSM}, \textbf{BIM}, \textbf{VNIFGSM}, \textbf{DIFGSM}.
The perturbation budget is set to $\epsilon=8/255$ with step size $2$ and number of steps $100$.
For our defense in adversarial training, the balance hyperparameter $\alpha$ in Eq. \eqref{eq:robust_local_obj} is set to $0.5$.

\subsection{RQ1: Verification of Theoretical Analysis}
\label{sec:exp_theory}

A core contribution of this paper is Theorem \ref{tm:transfer}, which posits that the transferability of adversarial examples is bounded by the model parameter difference between the attacker and victim (see Eq.~\eqref{eq:tran_proof}) as well as the distributional distance (see Eq.~\eqref{eq:final}).
To verify this, we conducted experiments calculating the trend in attack success rate (ASR) and the change of model prediction loss, which aligns with our analysis.
The results are visualized in Fig.~\ref{fig:correlation_all}.
It is obviously observed in Fig.~\ref{fig:svhn_fedprox},~\ref{fig:cifar10_fedprox}  that when the model parameter distance $\|w_V - w_A\|$ is increased, the adversarial attack ASR is reduced drastically, reflecting a negative correlation.
In Fig.~\ref{fig:svhn_transfer},~\ref{fig:cifar10_transfer}, the distribution distance is measured by the L1 norm of their class distribution difference $\|p_V - p_A\|$.
As displayed in the figures, although not a linear relation, the attack transferability gap is increased as the distribution distance gets large.
This positive correlation is measured in Spearman's coefficient, with $r$ values over 0.4 and 0.5, $p\approx0$ for both datasets, which means statistically significant correlation.

In summary, the empirical results are well-aligned with our analysis, confirming that the transferability of attacks and model/distribution difference are positively correlated.


\subsection{RQ2: Defense Performance Analysis}

We compare the proposed defense mechanism with baselines under four scenarios (IID vs.\ Non-IID on CIFAR-10 and SVHN). Results are reported in Tables~\ref{tab:results_cifar_iid}--\ref{tab:svhn_non}.

\textbf{Performance on IID Data.}
Tables~\ref{tab:results_cifar_iid} and~\ref{tab:svhn_iid} show that standard federated training without adversarial robustness (FedProx) is extremely vulnerable to transfer-based attacks: although clean accuracy remains high (e.g., 88--98\%), the accuracy on adversarial examples collapses to nearly zero (often $<1\%$) across all attack variants.
In contrast, both FAT and our method substantially improve robustness, and our method consistently provides additional gains over FAT while maintaining comparable clean accuracy.
Specifically, on CIFAR-10 in Table~\ref{tab:results_cifar_iid}, our method improves transfer robustness over FAT for both backbones.
For VGG11, ours increases robust accuracy by +1.69 (VMIFGSM), +1.17 (BIM), +1.79 (VNIFGSM), and +1.86 (DIFGSM) points, with only a minor clean-accuracy change (92.62 vs.\ 93.53).
For WRN28, ours yields larger gains, improving robustness by +3.36 (VMIFGSM), +2.04 (BIM), +3.21 (VNIFGSM), and +4.28 (DIFGSM) points while preserving clean accuracy (86.40 vs.\ 85.97).
Notably, our method is competitive with the centralized upper bound under transfer attacks for WRN28 (e.g., 83.93 vs.\ 83.79 on VMIFGSM), indicating that the proposed defense can approach centralized robustness even under federated constraints.
Table \ref{tab:svhn_iid} demonstrates the robustness of our method under IID settings on the SVHN dataset. 
Compared with the standard FedProx models (in VGG11 and WRN28), which show extreme vulnerability to transfer attacks, our defense method maintains robust accuracy on adversarial examples. 
For instance, in the VGG11 model, our method achieves an average robustness improvement of roughly $60\%$ and $6\%$ in BIM attack compared to FedProx and FAT, respectively.

\begin{table}[t]
\centering
\caption{Comparison of defense accuracy (\%) against various transfer-based adversarial attacks on \textbf{CIFAR-10 (IID)}.}
\label{tab:results_cifar_iid}
\resizebox{\columnwidth}{!}{%
\begin{tabular}{l|c|cccc}
\toprule
\textbf{Method} & \textbf{Clean} & \textbf{VMIFGSM} & \textbf{BIM} & \textbf{VNIFGSM} & \textbf{DIFGSM} \\
\midrule
\multicolumn{6}{c}{\textbf{VGG11}} \\
\midrule
Standard FedProx & 94.49 & $<1$ & $<1$ &$<1$ &$<1$ \\
FAT (Baseline) & 93.53 & 81.57 & 84.22 & 81.49 & 78.01 \\
\textbf{Ours} & \textbf{92.62} & \textbf{83.26} & \textbf{85.39} & \textbf{83.28} & \textbf{79.87}\\
Centralized (Upper Bound)& 95.68 & 94.57 & 93.78 & 93.43 & 94.96 \\
\midrule
\multicolumn{6}{c}{\textbf{WRN28}} \\
\midrule
Standard FedProx & 88.67 & $<1$ & $<1$ & $<1$ &  $<1$\\
FAT (Baseline) & 85.97 & 80.57 & 80.94 & 80.59 & 78.22 \\
\textbf{Ours} & \textbf{86.4} & \textbf{83.93} & \textbf{82.98} & \textbf{83.80} & \textbf{82.50}\\
Centralized (Upper Bound) & 88.76 & 83.79 & 84.08 & 83.78 & 80.78 \\
\bottomrule
\end{tabular}%
}
\end{table}

\begin{table}[t]
\centering
\caption{Comparison of defense accuracy (\%) against various transfer-based adversarial attacks on \textbf{SVHN (IID)}.}
\label{tab:svhn_iid}
\resizebox{\columnwidth}{!}{%
\begin{tabular}{l|c|cccc}
\toprule
\textbf{Method} & \textbf{Clean} & \textbf{VMIFGSM} & \textbf{BIM} & \textbf{VNIFGSM} & \textbf{DIFGSM} \\
\midrule
\multicolumn{6}{c}{\textbf{VGG11}} \\
\midrule
Standard FedProx & 97.78 & $<1$ & $<1$ & $<1$ & $<1$ \\
FAT (Baseline) & 95.69 & 43.13 & 54.13 & 42.93 & 43.05\\
\textbf{Ours} & \textbf{95.67} & \textbf{48.71} & \textbf{61.55} & \textbf{48.59} & \textbf{45.55} \\
Centralized (Upper Bound) & 98.33 & 90.20 & 93.44 & 89.86 & 84.75\\
\midrule
\multicolumn{6}{c}{\textbf{WRN28}} \\
\midrule
Standard FedProx & 96.50 & $<1$ & $<1$ &$<1$ & $<1$ \\
FAT (Baseline) & 95.69 & 54.90 & 60.43 & 55.11 & 54.93 \\
\textbf{Ours} & \textbf{95.50} & \textbf{56.48} & \textbf{63.95} & \textbf{56.90} & \textbf{57.67} \\
Centralized (Upper Bound) & 96.63 & 87.57 & 90.78 & 88.43 & 87.96 \\
\bottomrule
\end{tabular}%
}
\end{table}

\textbf{Performance on Non-IID Data.}
Tables~\ref{tab:results_cifar_non} and~\ref{tab:svhn_non} focus on the more realistic federated setting with heterogeneous client distributions.
Consistent with Theorem~\ref{tm:transfer}, transfer-based attacks tend to be slightly less effective under Non-IID partitions for non-robust baselines, due to larger distribution shifts between attacker and victim.
Nevertheless, the accuracy on adversarial examples remains significantly low ($<1\%$) without explicit defenses, whereas our method continues to outperform the baseline FAT across datasets.
For example, in Table~\ref{tab:results_cifar_non}, our method yields clear robustness gains over baselines.
For VGG11, it improves robustness over FAT by +5.06 (VMIFGSM), +4.59 (BIM), +4.97 (VNIFGSM), and +4.75 (DIFGSM) points.
For WRN28, ours improves both clean accuracy (+2.99 points) and robustness substantially, with +9.26, +9.07, +9.26, and +8.43 point gains under VMIFGSM/BIM/VNIFGSM/DIFGSM, respectively.
These results highlight that our method remains effective and can even be more beneficial under pronounced data heterogeneity.
Similar results can be observed on the SVHN non-i.i.d. scenario in Table~\ref{tab:svhn_non}.



\begin{table}[t]
\centering
\caption{Comparison of defense accuracy (\%) against various transfer-based adversarial attacks on \textbf{CIFAR-10 (Non-IID)}.}
\label{tab:results_cifar_non}
\resizebox{\columnwidth}{!}{%
\begin{tabular}{l|c|cccc}
\toprule
\textbf{Method} & \textbf{Clean} & \textbf{VMIFGSM} & \textbf{BIM} & \textbf{VNIFGSM} & \textbf{DIFGSM} \\
\midrule
\multicolumn{6}{c}{\textbf{VGG11 }} \\
\midrule
Standard FedProx &  87.01 & $<1$ & $<1$ & $<1$& $<1$ \\
FAT (Baseline) & 86.14 & 47.52 & 49.00 & 47.55 & 45.14\\
\textbf{Ours} & \textbf{83.56} & \textbf{52.58} & \textbf{53.59} & \textbf{52.52} & \textbf{49.89} \\
\midrule
\multicolumn{6}{c}{\textbf{WRN28}} \\
\midrule
Standard FedProx & 75.56 & $<1$ & $<1$ & $<1$ & $<1$ \\
FAT (Baseline) & 74.5 & 47.47 & 47.38 & 47.50 & 45.68 \\
\textbf{Ours} & \textbf{77.49} & \textbf{56.73} & \textbf{56.45} & \textbf{56.76} & \textbf{54.11} \\
\bottomrule
\end{tabular}%
}
\end{table}

\begin{table}[t]
\centering
\caption{Comparison of defense accuracy (\%) against various transfer-based adversarial attacks on \textbf{SVHN (Non-IID)}.}
\label{tab:svhn_non}
\resizebox{\columnwidth}{!}{%
\begin{tabular}{l|c|cccc}
\toprule
\textbf{Method} & \textbf{Clean} & \textbf{VMIFGSM} & \textbf{BIM} & \textbf{VNIFGSM} & \textbf{DIFGSM} \\
\midrule
\multicolumn{6}{c}{\textbf{VGG11}} \\
\midrule
Standard FedProx & 95.34 & $<1$ & $<1$ & $<1$ & $<1$ \\
FAT (Baseline) & 93.56 & 27.64 & 32.90 & 27.42 & 27.38\\
\textbf{Ours} & \textbf{94.05} & \textbf{28.95} & \textbf{36.30} & \textbf{28.75} & \textbf{28.48}\\
\midrule
\multicolumn{6}{c}{\textbf{WRN28}} \\
\midrule
Standard FedProx & 92.45 & $<1$ & $<1$ & $<1$ & $<1$ \\
FAT (Baseline) & 90.75 & 23.47 & 25.70 & 23.56 & 23.22 \\
\textbf{Ours} & \textbf{92.00} & \textbf{25.20} & \textbf{28.82} & \textbf{24.99} & \textbf{25.73} \\
\bottomrule
\end{tabular}%
}
\end{table}

\section{Conclusion}
\label{sec:con}


In this work, we studied transferable adversarial example attacks in federated learning and revealed how their effectiveness is fundamentally tied to the data distribution shift across clients. 
We provided a theoretical analysis that characterizes the vulnerability of federated systems under heterogeneous data. 
Building on these insights, we proposed a defense mechanism that leverages the SVD and robustness propagation.
Extensive experiments on real-world datasets demonstrate our method can improve robust accuracy while maintaining competitive clean performance, outperforming strong federated baselines.
Overall, our findings highlight transferability as a critical security dimension in federated learning and suggest that robustness-aware defenses can substantially strengthen FL in practical environments. 

\section*{Acknowledgment}

This work was supported by the National Science Foundation under grants No. 2429960, No. 2434899, No. 2548041, and No. 2343619.

\bibliographystyle{IEEEtran}
\bibliography{CIKM2022}
\end{document}